\documentclass[10pt,letterpaper]{article}
\usepackage[top=0.85in,left=2.75in,footskip=0.75in]{geometry}

\usepackage{amsmath,amssymb}

\usepackage{changepage}
\usepackage{graphicx}
\usepackage{epstopdf}
\usepackage{amsmath}
\usepackage{graphicx}
\usepackage{subcaption}
\usepackage{sectsty}

\usepackage{textcomp,marvosym}

\usepackage{nameref,hyperref}
\usepackage{booktabs}

\usepackage[right]{lineno}
\usepackage[backend=biber, style=numeric, sorting=none]{biblatex}
\usepackage[nopatch=eqnum]{microtype}
\DisableLigatures[f]{encoding = *, family = * }

\usepackage[table]{xcolor}

\usepackage{array}

\newtheorem{theorem}{Theorem}

\newenvironment{proof}{\textit{Proof}.}{\hfill$\square$}

\newcolumntype{+}{!{\vrule width 2pt}}

\newlength\savedwidth

\raggedright
\usepackage[aboveskip=1pt,labelfont=bf,labelsep=period,justification=raggedright,singlelinecheck=off]{caption}

\makeatletter
\renewcommand{\@biblabel}[1]{\quad#1.}
\makeatother

\usepackage{lastpage,fancyhdr,graphicx}
\usepackage{epstopdf}
\fancyheadoffset[L]{2.25in}
\fancyfootoffset[L]{2.25in}
\sectionfont{\fontsize{18}{22}\selectfont\bfseries}

\begin{document}
\vspace*{0.2in}

\begin{flushleft}
{\Large
\textbf\newline{SEGAN: A Semi-Supervised Learning Method for Missing Data Imputation} 
}
\newline
\\
Xiaohua Pan\textsuperscript{1,3},
Weifeng Wu\textsuperscript{2},
Peiran Liu\textsuperscript{3,7},
Guojun Sheng\textsuperscript{8},
Zhen Li\textsuperscript{1,4},
Peng Lu\textsuperscript{5,6*},
Peijian Cao\textsuperscript{5,6},
Jianfeng Zhang\textsuperscript{5,6},
Xianfei Qiu\textsuperscript{4},
Yangyang Wu\textsuperscript{2,3}
\\
\bigskip
\textbf{1} College of Computer Science and Technology, Zhejiang University, Hangzhou, Zhejiang, China
\\
\textbf{2} School of Software Technology, Zhejiang University, Hangzhou, Zhejiang, China
\\
\textbf{3} Binjiang Institute of Zhejiang University, Hangzhou, Zhejiang, China
\\
\textbf{4} China Academy of Space Technology, Beijing, China
\\
\textbf{5} Beijing Life Science Academy, Beijing, China
\\
\textbf{6} Zhengzhou Tobacco Research Institute of China National Tobacco Corporation, Zhengzhou, Henan, China
\\
\textbf{7} Purdue University, West Lafayette, Indiana, United States
\\
\textbf{8} COSMOPlat IoT Technology Co.,Ltd., Qingdao, Shandong, China 
\\
\bigskip

* Corresponding author\\
E-mail: penglu2004@hotmail.com\\

\end{flushleft}
\section*{Abstract}
In many practical real-world applications, data missing is a very common phenomenon, making the development of data-driven artificial intelligence theory and technology increasingly difficult. Data imputation is an important method for missing data preprocessing. Most existing missing data imputation models directly use the known information in the missing data set but ignore the impact of the data label information contained in the data set on the missing data imputation model.
In this paper, we propose a missing data imputation model based on semi-supervised learning, named SEGAN, which mainly includes three important modules: generator, discriminator, and classifier.
In the SEGAN model, the classifier enables the generator to make more full use of known data and its label information when predicting missing data values. In addition, the SEGAN model introduces a missing hint matrix to allow the discriminator to more effectively distinguish between known data and data filled by the generator. This paper theoretically proves that the SEGAN model that introduces a classifier and a missing hint matrix can learn the real known data distribution characteristics when reaching Nash equilibrium.
Extensive experiments on three public real-world datasets demonstrate that, SEGAN yields a more than 10\% accuracy gain, compared with the state-of-the-art approaches.



\section*{Introduction}
In the information age of today, the issue of data quality has become increasingly pivotal, with the problem of missing data being particularly widespread. This issue is not confined to any single sector but is prevalent across various industries such as healthcare, meteorology, and advertising, and it has a significant impact on data analysis and decision-making. Missing data can lead to incompleteness of information and may also cause biases in data analysis, thereby affecting the accuracy and reliability of decision-making. Faced with this challenge, the academic and industrial communities have devoted efforts to researching data imputation techniques, aiming to find more effective methods to address this issue~\cite{bib1}.

Historically, researchers have developed a variety of data imputation techniques to address the issue of missing data~\cite{bib44}~\cite{bib45}~\cite{bib46}~\cite{bib47}~\cite{bib48}~\cite{bib49}~\cite{bib50}. Early methods often relied on statistics, for instance filling in missing values by inserting mean, median, or mode~\cite{bib2}~\cite{bib3}. While these methods are straightforward in computation, they overlook the potential correlations and dynamic changes among data, often failing to achieve the desired imputation result. With the development of machine learning and deep learning methods, parametric imputation algorithms based on these methods have emerged~\cite{bib1}. These algorithms predict missing values by learning the complex relationships between data, demonstrating more accurate and efficient imputation capabilities.

Moreover, the information contained in data labels is crucial for research on data imputation and data analysis tasks~\cite{bib37} ~\cite{bib38}. However, current studies on data imputation only consider the impact of known data information on model prediction outcomes, neglecting the influence of data label information on imputation prediction research. Furthermore, in industrial applications, data label information also suffers from missing issues. For example, the PhysioBank archives store over 40GB of cardiac medical data, which includes some missing data and a limited amount of data label information~\cite{bib23}. Additionally, for legal reasons, hospitals retain a large volume of unlabeled cardiac data. This necessitates that data imputation algorithms not only fill in missing values but also consider the utilization of incomplete label information, thereby better serving downstream data analysis tasks.

To address these issues, this paper proposes a novel missing data imputation model based on semi-supervised learning, SEGAN, which utilizes both known data values and a subset of data labels to complete missing data. The SEGAN model primarily consists of three modules: a generator, a discriminator, and a classifier~\cite{bib24}. Specifically, the SEGAN model constructs a classifier based on semi-supervised learning, which drives the generator to complete missing values according to known data and data labels, and outputs the completed data. The discriminator aims to distinguish as much as possible between genuine data and the data completed by the generator. In addition, the SEGAN model introduces a hint matrix to the discriminator, encoding the missing status of part of the data. The main contributions are summarized as follows.

\begin{itemize}
\item We propose a missing data imputation model based on semi-supervised learning, SEGAN, which adopts a semi-supervised learning strategy to correct the imputation of incomplete data with partial labels. To the best of our knowledge, this model is the first to utilize data label information for data imputation studies.
\item Within the SEGAN model, the semi-supervised learning classifier is capable of inferring labels for unlabeled data, thereby requiring the generator to estimate or fill in missing data using existing data and data labels. Meanwhile, the model enhances the discriminator's accuracy in comparing known data with completed data by using the hint matrix.
\item We provide theoretical proof that the SEGAN model, based on the missing hint matrix, can effectively understand the information distribution of incomplete data.
\item Extensive experiments using several real-life datasets demonstrate the remarkable performance enhancement of SEGAN, compared with the state-of-the-art approaches.
\end{itemize}

The code is shared openly on Github (https://github.com/niunaicoke/SEGAN-A-Semi-Supervised-Learning-Method-for-Missing-Data-Imputation).

\section*{Related Work}
The imputation of incomplete structured data essentially requires predicting missing data based on known information, enhancing the integrity of structured data to prevent missing data issues from affecting subsequent decision-making analyses negatively. The simplest method to address missing data is to directly remove samples with missing data from the dataset. However, this approach is only suitable when the volume of incomplete data is small, and the removal of data can be ensured not to affect the analysis results ~\cite{bib28}. Such a method can easily reduce the information transmitted by the data and severely impact the outcomes of subsequent analysis. Therefore, to mitigate the problems caused by missing data, experts and scholars have focused on research into data imputation methods ~\cite{bib29}~\cite{bib30}~\cite{bib31}~\cite{bib32} ~\cite{bib33} ~\cite{bib34} ~\cite{bib35} ~\cite{bib36}. Their goal is to predict missing data based on known information, enhance data integrity, and ensure reliable and robust analytical results. Incomplete structured data imputation methods can be categorized into those based on statistical techniques, machine learning-based methods, and deep learning-based methods, depending on the type of prediction model used.

More specifically, statistical techniques for data imputation can be categorized based on different statistical mechanisms into methods that rely on statistical information and those based on similarity~\cite{bib27} ~\cite{bib43}~\cite{bib42}. Methods that use statistical information, such as mean, median, or mode imputation algorithms~\cite{bib2}, estimate missing values by utilizing statistical data descriptors. Additionally, similarity-based imputation algorithms predict missing data for a target sample using the mean of known values from multiple similar samples. This category includes algorithms such as the K-Nearest Neighbor Imputation (KNNI)~\cite{bib3} and Cold Deck Imputation (CDI)~\cite{bib4}.

Machine learning based imputation methods involve training one or more predictive models to estimate and fill in missing values. Depending on the machine learning model utilized, these methods can be classified into three types: decision tree-based, linear regression-based, and data compression-based imputation algorithms~\cite{bib25}. Decision tree-based imputation algorithms construct a predictive decision tree model for each feature with missing data. Examples include the XGBoost Imputation (XGBI)~\cite{bib5} algorithm, which leverages the XGBoost framework, and the MissForest Imputation (Miss-FI)~\cite{bib6} algorithm, which is based on random forests. Linear regression-based imputation algorithms build multiple linear regression models for each feature with incomplete data, for instance, the Multivariate Imputation by Chained Equations (MICE) ~\cite{bib7} algorithm, which employs a series of chained equations, and the Imputation via Individual Model (IIM)~\cite{bib8}, which constructs individual regression models for each feature. Data compression-based imputation algorithms create a data compression predictive model for the entire set of missing data, predicting missing values using singular value decomposition and data reconstruction strategies. This includes the Soft-impute (SI)~\cite{bib9} algorithm, which applies a soft-thresholding approach, the Matrix Factorization Imputation (MFI) ~\cite{bib10}, which factors the data matrix, and the Principal Component Analysis Imputation (PCAI)~\cite{bib11}, which uses PCA for data imputation.

Deep learning-based imputation methods vary according to the specific deep learning models employed and primarily include algorithms based on multi-layer perceptrons, autoencoders, and generative adversarial networks~\cite{bib24}. Multi-layer perceptron-based imputation algorithms construct a predictive model using a multi-layer perceptron (MLP) for each incomplete feature. This includes the Multi-Layer Perceptron Imputation (MLPI)~\cite{bib12} and the Round-Robin Sinkhorn Imputation (RRSI)~\cite{bib13} algorithms. Autoencoder-based imputation algorithms compress incomplete input data into a shallow vector, which is then reconstructed by the decoder into a matrix that closely resembles the input data matrix. Autoencoder-based imputation methods include the Multiple Imputation Denoising Auto-Encoder (MIDAE)~\cite{bib14}, the Variational Auto-Encoder Imputation (VAEI)~\cite{bib15}, the Heterogeneous-Incomplete VAE (HIVAE)~\cite{bib16}, and the Missing Data Importance-Weighted Auto-Encoder Model (MIWAE)~\cite{bib17}. Generative adversarial network-based imputation methods design an imputation module to generate missing data that closely approximates the true known distribution and construct a discriminator module to differentiate between imputed and genuine data as accurately as possible~\cite{bib26}. This category includes the Generative Adversarial Imputation Network (GAIN)~\cite{bib18}, the Graph Imputation Neural Network (GINN) ~\cite{bib19}, the Efficient and Effective Data Imputation System with Influence Functions (EDIT)~\cite{bib20}, and the Scalable Imputation System (SCIS)~\cite{bib21}, which estimates the number of samples for generative adversarial imputation.

In summary, existing imputation methods for incomplete structured data focus solely on data imputation and its processing, neglecting the impact of the imputed results on downstream decision-making outcomes, which leads to a lack of effective support of subsequent decision analysis~\cite{bib22}.





\section*{Methods}
\subsection*{Problem Definition}
An incomplete dataset comprises a set of multivariate sample collections $S=\left(\mathbf{X}_1,\cdots,\mathbf{X}_s\right)$ with labels $Y=\left(\mathbf{y}_1,\cdots,\mathbf{y}_s\right)$. Each sample $\mathbf{X}$ in $S$ corresponds to a data label $y$. Formally, $\mathbf{X} =\left(X_1,\cdots,X_n\right)\in\mathbb{R}^{d\times n}$, where $X_i=\left(x_i^1,\cdots,x_i^d\right)^T$, and $x_i^j$ denotes the $j$-th feature value of the $i$-th data point in $\mathbf{X}$. For each sample $\mathbf{X}$, a mask matrix $M =\left(M_1,\cdots,M_n\right)\in\mathbb{R}^{d\times n}$ encodes the state of missing data within $\mathbf{X}$. In such state, the mask vector $M_i=\left(m_i^1,\cdots,m_i^d\right)^T$, where $m_i^j =0$ indicates that $x_i^j$ is missing, and otherwise, it indicates existence.

In summary, this paper focuses on addressing the issue of imputing missing multivariate data in an incomplete dataset $S$. The goal is to predict and fill in an appropriate value for each missing entry in $S$ to achieve two main objectives: (1) to refine the imputed multivariate dataset, denoted as $\hat{S}$, to approximate the true complete dataset as closely as possible, assuming such a dataset exists; (2) to ensure that the downstream predictive models based on the imputed multivariate dataset achieve more accurate multivariate data analysis performance compared to models that only use the original incomplete multivariate dataset.

\subsection*{Model Overview}
This section provides an introduction to the architecture of the missing data imputation model SEGAN, which is based on semi-supervised learning. Figure ~\ref{fig1} presents the overall structure of the SEGAN model. The input data include the incomplete dataset, data-label pairs, and the mask matrix. The SEGAN model consists of three main modules: the generator $G$, the discriminator $D$, and the classifier $C$.

\begin{figure}[!h]
\includegraphics[width=0.9\textwidth]{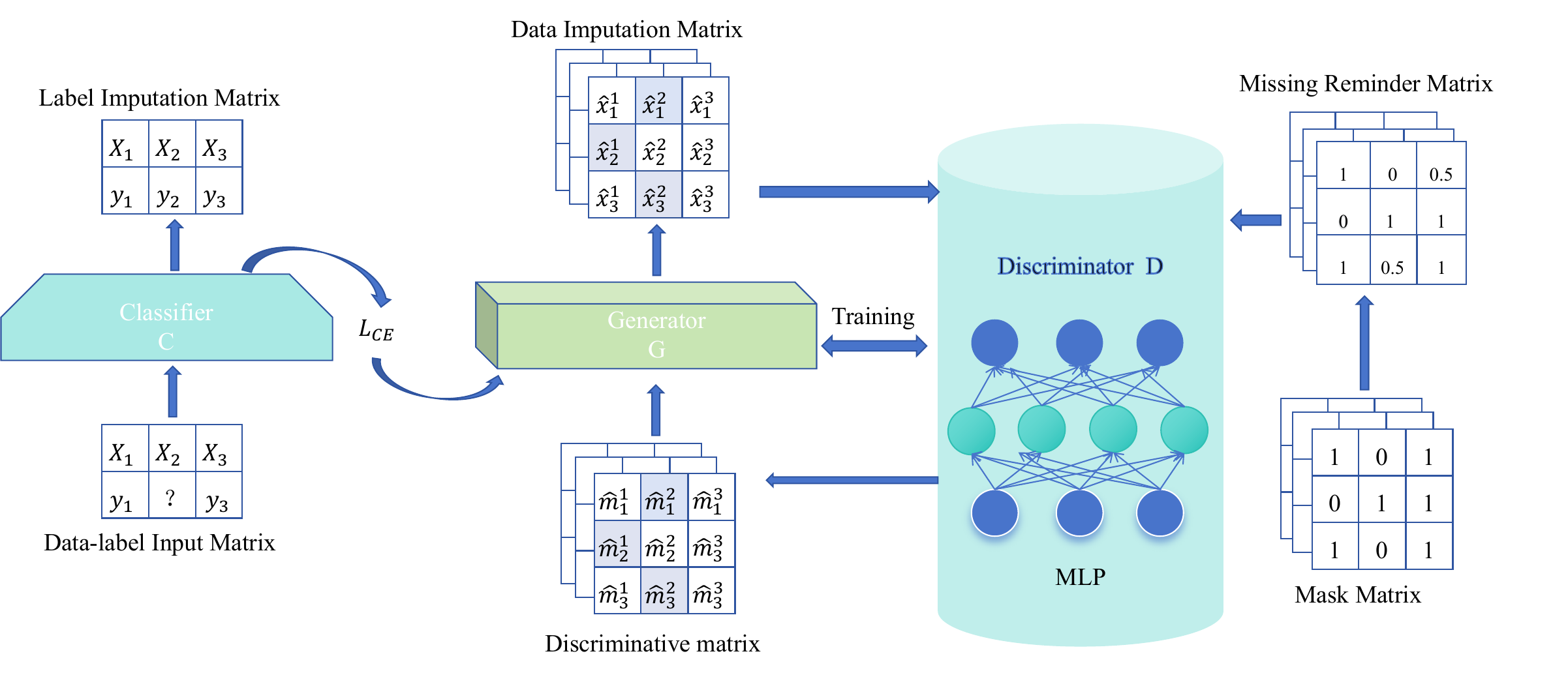}
\caption{
The architecture of SEGAN.}
\label{fig1}
\end{figure}

The generator $G$ utilizes known data and data labels to predict missing values, creating a completed data matrix with the aim of "deceiving" the discriminator $D$. Meanwhile, the classifier $C$ is trained using labeled multivariate data to predict labels for unlabeled data. The classifier $C$ backpropagates the cross-entropy loss function $L_{\text{CE}}$ to the generator $G$, guiding it to pay more attention to multivariate data samples with the same label when repairing incomplete samples.

The discriminator $D$ differentiates between the data completed by the generator and the genuine data based on the completed data and the hint matrix. The hint matrix encodes the state of partial data missing, stored in the mask matrix. The discriminator outputs a probability matrix that contains the likelihood of each data point being genuine. The design of this structure aims to enhance the model's effectiveness in completing incomplete data.

Furthermore, the SEGAN model integrates multi-layer perceptrons (MLPs) with its three primary modules. An MLP is a classical neural network model comprising multiple layers of neurons.
Given a multivariate data sample $\mathbf{X}$, the MLP unit outputs a reconstructed data matrix $\overline{X}$. The update steps for the MLP unit are as follows:
\begin{equation}
    H = \alpha(W_x \mathbf{X} + \mathbf{b}_x), \quad \overline{\mathbf{X}} = W_h H + \mathbf{b}_h \quad
\end{equation}
where each sample has $d$ features, thus the input dataset $\mathbf{X} \in \mathbb{R}^{d \times n}$. Assuming a single hidden layer MLP with $h$ hidden units, the hidden layer weights would be $W_x \in \mathbb{R}^{h \times d}$ and biases $\mathbf{b}_x \in \mathbb{R}^{h \times 1}$, resulting in the hidden layer output $H \in \mathbb{R}^{h \times n}$. Since there are $q$ categories, the output layer weights would be $W_h \in \mathbb{R}^{q \times h}$ and biases $\mathbf{b}_h \in \mathbb{R}^{q \times 1}$, with $\alpha$ representing the activation function.

\subsection*{SEGAN Model}

This section will discuss the three modules of the SEGAN model and provide theoretical proof of the SEGAN model based on the hint matrix and the classifier.

\subsubsection*{Generator}

The generator is constructed using MLP units, with its primary objective being to impute missing parts of multivariate data by learning the distribution and labels of the known genuine multivariate data. In essence, the generator takes the data matrix $\mathbf{X}$ and the mask matrix $\mathbf{M}$ as inputs and utilizes the SEGAN model to produce the imputed data.
The objective function of the generator primarily comprises three types of loss functions: the adversarial loss, the classifier loss, and the reconstruction loss. The adversarial loss is identical to that in Generative Adversarial Networks~\cite{bib24}; the classifier loss arises from the classifier's feedback on label prediction, as specifically defined in Equation (10); the goal of the reconstruction loss is to align the known multivariate time-series data as closely as possible with the data reconstructed by the generator, which is mathematically expressed as follows:

\begin{equation}
    L_{X,\overline{X}} = \mathbf{M} \odot L_e(\mathbf{X}, \overline{\mathbf{X}}) \quad
\end{equation}

Here, $\overline{\mathbf{X}}$ represents the reconstructed multivariate data matrix, and $L_e$ is the Mean Absolute Error (MAE) function, which is used to measure the discrepancy between the predicted values and the ground truth values.
The final objective function for the generator $L_G$ is as follows:

\begin{equation}
    L_G = \frac{1}{n} L_{X,\overline{X}} + \alpha \mathbf{E} \left[ (1 - \mathbf{M}) \odot \log(1 - \mathbf{D}(\mathbf{X}, \mathbf{R})) \right] + \beta L_C(\mathbf{y}, \overline{\mathbf{X}}) \quad
\end{equation}

Here, $n$ denotes the number of sample data points. $L_C(\mathbf{y}, \hat{\mathbf{X}})$ is the classifier's loss function, where $\mathbf{y}$ represents the labels of the input data. The term $\mathbf{E}\left[(1 - \mathbf{M}) \odot \log(1 - \mathbf{D}(\hat{\mathbf{X}}, \mathbf{R}))\right]$ is the specific adversarial loss function for the generator and discriminator during the training process of SEGAN. $\hat{\mathbf{X}}$ is the data output completed by the generator, and the definition of $\hat{\mathbf{X}}$ is as follows:

\begin{equation}
    \hat{\mathbf{X}} = \mathbf{M} \odot \mathbf{X} + (1 - \mathbf{M}) \odot \overline{\mathbf{X}} \quad
\end{equation}

The generator adjusts its parameters with backpropagation, based on the gradient updated from the discriminator and classifier to minimize the generator's objective function $L_G$. Such an iterative process will gradually enable the generator to produce more genuine-like data, adapting to the enhanced capabilities of the discriminator and classifier.

\subsubsection*{Discriminator}

The SEGAN model adopts the traditional concept of generative adversarial models~\cite{bib24}, motivating the generator to learn the genuine data distribution through the optimization process of the generator and discriminator. Compared to a standard discriminator, the discriminator in the SEGAN model produces a probability matrix, where each probability value indicates the likelihood that each value in the multivariate data $\hat{\mathbf{X}}$, completed after imputation, is genuine. The goal of the discriminator is to differentiate as much as possible between the genuine data and the generated data in $\hat{\mathbf{X}}$. Inspired by the GAIN algorithm~\cite{bib6}, the SEGAN model introduces a concept known as the ``missingness hint matrix'' $\mathbf{R}$, which stores information about the missing status of certain data parts and can be defined as follows:
\begin{equation}
    \mathbf{R} = \mathbf{K} \odot \mathbf{M} + 0.5(1 - \mathbf{K}) \quad
\end{equation}
where $\mathbf{K} = (\mathbf{k}_1, \cdots, \mathbf{k}_n) \in \{0, 1\}^{d \times n}$ is a random component matrix for each training iteration, and $\mathbf{R}$ provides partial missing data to facilitate the rapid convergence of the discriminator.

The discriminator $\mathbf{D}$, composed of MLP (Multi-Layer Perceptron) units, takes the generator-completed $\hat{\mathbf{X}}$ and the missingness hint matrix $\mathbf{R}$ as inputs and outputs a discrimination matrix $\hat{\mathbf{M}}$, where $\hat{\mathbf{M}} = (\hat{m}_1, \cdots, \hat{m}_n)$. The discrimination probabilities $\hat{m}_i$ within the discrimination matrix represent the likelihood that each value belongs to the real data, which can be defined as follows:
\begin{equation}
    \mathbf{m}_i = \mathbf{W}_x \mathbf{x}_i + \mathbf{b}_x
\end{equation}
$\mathbf{W}_x$ and $\mathbf{b}_x$ are the parameters of the discriminator, and $\hat{\mathbf{M}} = (\hat{m}_1, \cdots, \hat{m}_n) \in \mathbb{R}^{d \times n}$. Initially, the discriminator utilizes MLP units to extract features from the data and outputs the probability that each value in the multivariate data belongs to the genuine values. Ultimately, the probability matrix output by the discriminator, denoted as $\mathbf{D}(\hat{\mathbf{X}}, \mathbf{R})$, is expressed as follows:
\begin{equation}
    \mathbf{D}(\hat{\mathbf{X}}, \mathbf{R}) = \mathbf{M} = (m_1, \cdots, m_n) \quad 
\end{equation}

In summary, the loss function for the discriminator is defined as follows:

\begin{equation}
    L_D = -\mathbf{E}[\mathbf{M} \odot \log(\mathbf{M}) + (1 - \mathbf{M}) \odot \log(1 - \mathbf{M})] \quad 
\end{equation}

The process of minimizing the discriminator's loss function aims to enhance the accuracy of the discriminator, enabling it to produce probability matrix predictions for the imputed data that are closer to the genuine distribution. This is done to make the probability matrix $\mathbf{D}$ predicted by the discriminator as close as possible to the mask matrix $\mathbf{M}$.

\subsubsection*{Classifier}
In traditional Generative Adversarial Networks (GANs)~\cite{bib9}~\cite{bib10}~\cite{bib24}, the discriminator has the dual task of distinguishing between real and generated data as well as classifying the predicted data. However, the adversarial training in traditional GANs under a semi-supervised learning setting has its limitations~\cite{bib9}. This is due to the theoretical existence of two incompatible model convergence points. The task of distinguishing generated data is often simpler than the task of predicting labels, which can lead the discriminator to focus more on identifying generated data, thereby reducing its performance in label prediction.

To address this issue, the SEGAN model constructs a data classification module based on semi-supervised learning to predict the label information of data. This classification system is similar to the discriminator module and is composed of multiple MLP (Multi-Layer Perceptron) units. To reduce the number of model parameters and increase training speed, the model allows the classifier and the discriminator module to share the parameters of the MLP units. With the self-training strategy of semi-supervised learning, the classifier is capable of effectively learning and simultaneously performing classification labeling on unlabeled data.

During the training process of the SEGAN model, the classifier is trained using labeled data generated by the generator, while also assigning labels to "unlabeled" data samples that are predicted with high confidence. This approach encourages the generator to focus more on multivariate time series data with consistent labels, thereby optimizing the generator's output.

The ultimate goal of the classifier is to identify an optimal set of classifier parameters that will make the highest confidence predictions on the true labels of multivariate time series data. The classifier's prediction result $\mathbf{C}(\mathbf{X})$ can be expressed as:
\begin{equation}
    \mathbf{C}(\mathbf{X}) = f(\mathbf{W}_h\mathbf{H} + \mathbf{b}_h) \quad 
\end{equation}
From this, we can define the classifier loss function $\mathbf{L}_\mathbf{C}$:
\begin{equation}
    \mathbf{L}_\mathbf{C}(\mathbf{y}, \mathbf{X}) = \mathbf{L}_{\mathbf{CE}}(\mathbf{y}, \mathbf{C}(\mathbf{X})) \quad 
\end{equation}
Here, $\mathbf{L}_{\mathbf{CE}}$ represents the cross-entropy loss function, and $f$ denotes the softmax function. Thus, the SEGAN model employs the discriminator and classifier for their respective tasks of distinguishing generated data and predicting data labels. During the SEGAN model update process, the generator receives classification feedback from the classifier, which allows it to fully consider the corresponding feature set in multivariate time series data with the same label when imputing incomplete multivariate data. Concurrently, the classifier continues to train the model using the multivariate time series data filled in by the generator. Therefore, the objective functions of both the generator and classifier in the SEGAN model are inclined to converge simultaneously, effectively addressing the model convergence issue and leading to more precise imputations by the generator.

\subsection*{Theoretical Analysis}

The objective function of the SEGAN model can be defined as a minimax game function:

\begin{align}
\min_{G,C} \max_{D} \ U(G,D,C) &= \mathbb{E}_{\hat{\mathbf{X}},\mathbf{M},\mathbf{R}}\left[\mathbf{M} \odot \log D\left(\hat{\mathbf{X}},\mathbf{R}\right) + (1 - \mathbf{M}) \odot \log(1 - D\left(\hat{\mathbf{X}},\mathbf{R}\right))\right] \nonumber \\
&\quad + \beta L_{\text{CE}}(\mathbf{y}, C(\mathbf{X})) \nonumber \\
&= V(G,D) + \beta L(G,C) 
\end{align}

Here, $\hat{\mathbf{X}}$ represents the multivariate time series data samples filled in by the generator $\mathbf{G}$, $D(\hat{\mathbf{X}},\mathbf{R})$ is the discrimination result from discriminator $\mathbf{D}$, and $C(\hat{\mathbf{X}})$ is the prediction result from classifier $\mathbf{C}$. Other notations such as $\mathbf{M}$ denote the mask matrix for the incomplete data $\mathbf{X}$, $L_{CE}$ is the cross-entropy loss function, $\mathbf{y}$ is the true label of $\hat{\mathbf{X}}$, and $\beta$ represents a hyperparameter.

The SEGAN model $U(G, D, C)$ can be regarded as the USEGAN model $V(G, D)$ with an added classifier $C$. The distribution of the random variables $(\hat{\mathbf{X}},\mathbf{M},\mathbf{R})$ can be defined as $p(\hat{\mathbf{X}},\mathbf{M},\mathbf{R})$, while the marginal distributions of $\hat{\mathbf{X}}$, $\mathbf{M}$, and $\mathbf{R}$ can be respectively defined as $\hat{p}$, $p_{\mathbf{m}}$, and $p_{\mathbf{r}}$.

According to minimax game theory, the SEGAN model, which incorporates a missingness hint matrix and a classifier, will have its generated data distribution and predicted label distribution converges to the true data distribution and true label distribution upon reaching Nash equilibrium. Firstly, Theorem 1 demonstrates that the missingness hint matrix can aid in determining the equilibrium points of the USEGAN model $V(G, D)$. Subsequently, Theorem 2 proves that when converging to Nash equilibrium, the data distribution generated by $G$ and the label distribution predicted by $C$ will each coincide with the true data distribution and true label distribution, respectively. It is noteworthy that the theorems below are predicated on the assumption of relative independence between $\mathbf{M}$ and $\mathbf{X}$, meaning that the multivariate data is Missing Completely At Random (MCAR) \cite{bib25}.

\begin{theorem}
\label{theorem:USEGAN}
For a given missingness hint matrix $\mathbf{R}$, the equilibrium state of the USEGAN model $V(G, D)$ can be uniquely determined. At this equilibrium, the data distribution $p$ generated by the USEGAN model and the true data distribution $\hat{p}$ are the same under the given conditions, that is, $p(\mathbf{X}|\mathbf{M}) = p(\mathbf{X}|\mathbf{1})$.
\end{theorem}

\begin{proof}
The proof of Theorem \ref{theorem:USEGAN} consists of three parts and extends the research of the hint mechanism [6] in the imputation of multivariate data: First, it establishes the existence of an optimal discriminator $\mathbf{D}^*$ for the USEGAN model $V(G, D)$ when the generator $G$ is fixed; second, it demonstrates that a discriminator $\mathbf{D}^*$ without the missingness hint matrix $\mathbf{R}$ would lead to $G$ generating multiple data distributions; and finally, it proves that the missingness hint matrix $\mathbf{R}$ ensures that $G$ generates a unique true data distribution.

Let $\mathbf{M}_\mathbf{e}^{\mathbf{ji}}$ denote the set where $\left\{\mathbf{M} \in \{0,1\}^{d \times n} : m_i^j = e\right\}$ for the optimal discriminator. Consequently, there are two sets, $\mathbf{M}_1^{\mathbf{ji}}$ and $\mathbf{M}_0^{\mathbf{ji}}$, which satisfy the following conditions:

\begin{align}
V(G,D) &= \int_{\mathcal{X}} \sum_{\mathbf{M} \in \{0,1\}^{d \times n}} \int_{\mathcal{R}} \left(\mathbf{M} \odot \log D(\mathbf{X},\mathbf{R}) \right. \nonumber \\
&\quad \left. + (1-\mathbf{M}) \odot \log(1-D(\mathbf{X},\mathbf{R}))\right) p(\mathbf{X},\mathbf{M},\mathbf{R}) \, d\mathbf{R} \, d\mathbf{X} \nonumber \\
&= \int_{\mathcal{X}} \int_{\mathcal{R}} \sum_{\mathbf{M} \in \{0,1\}^{d \times n}} \left(\sum_{(j,i) : m_i^j=1} \log D(\mathbf{X},\mathbf{R})_i^j \right. \nonumber \\
&\quad \left. + \sum_{(j,i) : m_i^j=0} \log(1-D(\mathbf{X},\mathbf{R})_i^j)\right) p(\mathbf{X},\mathbf{M},\mathbf{R}) \, d\mathbf{R} \, d\mathbf{X} \nonumber \\
&= \int_{\mathcal{X}} \int_{\mathcal{R}} \sum_{i=1}^{n} \sum_{j=1}^{d} \left(\sum_{\mathbf{M} \in \mathbf{M}_1^{\mathbf{ji}}} \log D(\mathbf{X},\mathbf{R})_i^j \right. \nonumber \\
&\quad \left. + \sum_{\mathbf{M} \in \mathbf{M}_0^{\mathbf{ji}}} \log(1-D(\mathbf{X},\mathbf{R})_i^j)\right) p(\mathbf{X},\mathbf{M},\mathbf{R}) \, d\mathbf{R} \, d\mathbf{X}
\end{align}

The mapping of the function $y$ is defined as $y \mapsto a \log y + b \log(1 - y)$, and its maximum value within the interval $[0, 1]$ is defined as $\frac{a}{a + b}$. Therefore, when the generator $G$ is fixed, if the discriminator $D$ exists in the following manner, then the USEGAN model $V(G, D)$ will achieve its maximum value.

\begin{equation}
D(\hat{\mathbf{X}}, \mathbf{R})_{ij} = \frac{p(\hat{\mathbf{X}}, \mathbf{R} | m_{ij} = 1)}{p(\hat{\mathbf{X}}, \mathbf{R} | m_{ij} = 1) + p(\hat{\mathbf{X}}, \mathbf{R} | m_{ij} = 0)} = p(m_{ij} = 1 | \hat{\mathbf{X}}, \mathbf{R}) \tag{13}
\end{equation}
Here, $i \in \{1, \cdots, n\}$, $j \in \{1, \cdots, d\}$. The variable $p$ denotes the distribution of the random variable $(\hat{\mathbf{X}}, \mathbf{M}, \mathbf{R})$, with $\mathbf{p}_m$ being the marginal distribution.
In the following section, we will derive the global optimum of $\mathbf{V}(\mathbf{G}, \mathbf{D})$ based on $\mathbf{D}^*$. Clearly, if the missingness hint matrix $\mathbf{R}$ is not provided to $\mathbf{D}^*$, then the generator $\mathbf{G}$ will create multiple data distributions, elevating $\mathbf{V}(\mathbf{G}, \mathbf{D})$ to an equilibrium state. With the help of $\mathbf{D}^*$, we can redefine $\mathbf{V}(\mathbf{G}, \mathbf{D})$ as follows:
\begin{equation}
O(G) = \mathbb{E}_{\mathbf{X}, \mathbf{M}, \mathbf{R}} \left[ j, i : m_{ij} = 1 \log p(m_{ij} = 1 | \mathbf{X}, \mathbf{R}) + j, i : m_{ij} = 0 \log p(m_{ij} = 0 | \mathbf{X}, \mathbf{R}) \right] \tag{14}
\end{equation}

Let $\mathbf{R}_e^{ji}$ represent the set where $R \in \mathbf{R} : p_r(\mathbf{R} | m_i^j = e) > 0$, with $p_r$ referring to the marginal distribution of $R$. Thus, it can be understood as:
\begin{equation}
\begin{aligned}
O(G) &= \sum_{i=1}^{n} \sum_{j=1}^{d} \sum_{e \in \{0, 1\}} \left( \int_{\mathbf{R}_e^{ji}} \int_{\mathbf{X}} p(\mathbf{X}, \mathbf{R}, m_i^j = e) \log p(m_i^j = e | \mathbf{X}, \mathbf{R}) \, dR \, dX \right) \\
&= \sum_{i=1}^{n} \sum_{j=1}^{d} \sum_{e \in \{0, 1\}} \int_{\mathbf{R}_e^{ji}} \int_{\mathbf{X}} p(\mathbf{X}, \mathbf{R}, m_i^j = e) \log p(m_i^j = e | \mathbf{R}) \\
&\quad+ \log \left( \frac{\hat{p}(\mathbf{X} | m_i^j = e, \mathbf{R})}{\hat{p}(\mathbf{X} | \mathbf{R})} \right) \, dR \, dX \\
&= \sum_{i=1}^{n} \sum_{j=1}^{d} \sum_{e \in \{0, 1\}} \int_{\mathbf{R}_e^{ji}} p(m_i^j = e, \mathbf{R}) D_{KL} \left( \hat{p}(\cdot | \mathbf{R}, m_i^j = e) \| \hat{p}(\cdot | \mathbf{R}) \right) \, dR
\end{aligned} \tag{15}
\end{equation}

Here, the validity of the third equation relies on the logarithmic property that $\log(ab) = \log(a) + \log(b)$, where $D_{KL}$ represents the Kullback-Leibler divergence. Given that the Kullback-Leibler divergence between two distributions is always non-negative and only equals zero when the two distributions are identical, it can thus be observed that:
\begin{equation}
O^* = \sum_{i=1}^{n} \sum_{j=1}^{d} \sum_{e \in \{0, 1\}} \int_{\mathbf{R}_e^{ji}} p(m_{ij}=e | \mathbf{R}) \log p(m_{ij}=e | \mathbf{R}) \, dR \tag{16}
\end{equation}
is the global minimum of $O(G)$, and the solution satisfies the following conditions:
\begin{equation}
p(\cdot | \mathbf{R}, m_{ij}=e) = p(\cdot | \mathbf{R}) \tag{17}
\end{equation}
It is important to emphasize that when $\mathbf{R}$ and $M$ are independent, Equation (17) can be rewritten as $\hat{p}(\mathbf{X} | m_{i}^{j}=e) = \hat{p}(\mathbf{X})$. Here, $\hat{p}(\mathbf{X} | m_{i}^{j}=e) = \hat{p}(\mathbf{X})$, and $e \in \{0, 1\}$. This relationship holds true only under the following conditions:
\begin{equation}
p(\mathbf{X} | m_{ij}=1) = p(\mathbf{X} | m_{ij}=0) \tag{18}
\end{equation}
For $i \in \{1, \cdots, n\}$, $j \in \{1, \cdots, d\}$, since the number of linear equations defined by Equation (18) is less than the number of parameters specified by $G$, there exist multiple data distributions that satisfy Equation (17). In the final stage, the paper demonstrates that the missingness hint matrix $\mathbf{R}$, which records partial information of $\mathbf{M}$, can drive the generator $G$ to learn a unique expected data distribution. This is the case if $\mathbf{M}^0, \mathbf{M}^1 \in \{0, 1\}^{d \times n}$ differ by only one element. In such a scenario, the $(j, i)$-th element of $\mathbf{M}^0$ is 0, and the $(j, i)$-th element of $\mathbf{M}^1$ is 1. Consequently, the missingness hint matrix $\mathbf{R}$, associated with the mask matrix $\mathbf{M}$, satisfies $p_r(\mathbf{R} | m_{i}^{j}=e) > 0$. When $e=0$ and $e=1$:

\begin{equation}
p(\mathbf{X} | \mathbf{R}, m_{ij}=e) = p(\mathbf{X} | \mathbf{M}_e, \mathbf{K}) = p(\mathbf{X} | \mathbf{M}_e) p(\mathbf{K} | \mathbf{M}_e) = p(\mathbf{X} | \mathbf{M}_e) p(\mathbf{K}) \tag{19}
\end{equation}
Here, the matrix $\mathbf{K}$ assists $\mathbf{R}$ in encoding partial information of $\mathbf{M}$. The marginal distribution of $\mathbf{K}$ is denoted by $p_{\mathbf{k}}$. The first equation is based on the premise that the information from $\mathbf{M}$ and $\mathbf{K}$ is encapsulated within $\mathbf{R}$. The validity of the last equation arises from the independence between $\mathbf{K}$ and $\mathbf{M}$. Therefore, according to Equation (17), it can be concluded:
\begin{equation}
p(\mathbf{X} | \mathbf{M}^0) = p(\mathbf{X} | \mathbf{M}^1) \tag{20}
\end{equation}
Equation (20) is applicable for $\mathbf{M}^0$ and $\mathbf{M}^1$ which differ only in one variable. Let's assume $\mathbf{M}^1$ and $\mathbf{M}^2$ are any two matrices in $\{0,1\}^{d \times n}$. Therefore, there exists a natural number $z$ and a series of matrices $\mathbf{M}^{1\prime}, \cdots, \mathbf{M}^{z\prime}$ such that $\mathbf{M}^{i\prime}$ and $\mathbf{M}^{(i+1)\prime}$ differ only in one specific variable, with $\mathbf{M}^1 = \mathbf{M}^{1\prime}$, $\mathbf{M}^2 = \mathbf{M}^{z\prime}$, thus:
\begin{equation}
p(\mathbf{X} | \mathbf{M}^1) = p(\mathbf{X} | \mathbf{M}^{1\prime}) = \cdots = p(\mathbf{X} | \mathbf{M}^{z\prime}) = p(\mathbf{X} | \mathbf{M}^2) \tag{21}
\end{equation}
For any $\mathbf{M} \in \{0,1\}^{d \times n}$, it holds true that:
\begin{equation}
p(\mathbf{X} | \mathbf{M}) = p(\mathbf{X} | \mathbf{1}) \tag{22}
\end{equation}
Given that $\hat{p}(\mathbf{X} | \mathbf{1})$ represents the true and unique data distribution, it follows that any distribution that satisfies Equation (17) must be unique. In other words, the equilibrium position of the USEGAN model $V(G, D)$, which depends on $\mathbf{R}$, can be uniquely determined. At this point, the data distribution $\hat{p}(\mathbf{X} | \mathbf{M})$ generated by the generator $G$ will be consistent with the actual data distribution.
\end{proof}


\begin{theorem}
The SEGAN model $U(G, D, C)$ achieves Nash equilibrium under the condition that $\hat{p}(\mathbf{X} | \mathbf{M}) = \hat{p}(\mathbf{X} | \mathbf{1})$ and $p_c(\cdot | \mathbf{X}) = p_l(\cdot | \mathbf{X})$. Here, $p_c(\cdot | \mathbf{X})$ denotes the predicted data label distribution by the model given $\mathbf{X}$, while $p_l(\cdot | \mathbf{X})$ corresponds to the true data label distribution given $\mathbf{X}$.
\end{theorem}

\begin{proof}
In the SEGAN model $U(G, D, C) = V(G, D) + \beta L(G, C)$, the term $L(G, C)$ can be defined as:
\begin{align}
L(G, C) &= L_{CE}(y, C(\hat{X})) = \mathbb{E}_{\hat{X} \sim G(X)}\left[ \mathbb{E}_{y' \sim p_l(y | \hat{X})}\left[ -\log p_c(y' | \hat{X}) \right] \right] \nonumber \\
&= \mathbb{E}_{\hat{X} \sim G(X)}\left[ D_{KL}(p_l(\cdot | \hat{X}) \| p_c(\cdot | \hat{X})) \right] \nonumber \\
&+ \mathbb{E}_{y' \sim p_l(y | \hat{X})}\left[ -\log p_l(y' | \hat{X}) \right] \tag{23}
\end{align}

Here, the second equation arises from the definition of the cross-entropy loss function, that is, $L_{CE}(q, q') = -\sum_{x} q(x) \log q'(x)$. This means that for the SEGAN model, minimizing $L(G, C)$ is equivalent to minimizing the Kullback-Leibler divergence $D_{KL}(p_l(\cdot | \hat{X}) \| p_c(\cdot | \hat{X}))$. The latter is non-negative and can only be zero when $p_l(\cdot | \hat{X}) = p_c(\cdot | \hat{X})$. Furthermore, according to Theorem 1, it can be seen that when the model reaches equilibrium, $\hat{p}(\mathbf{X} | \mathbf{M}) = \hat{p}(\mathbf{X} | \mathbf{1})$. Therefore, it is only when the SEGAN model reaches Nash equilibrium that $\hat{p}(\mathbf{X} | \mathbf{M}) = \hat{p}(\mathbf{X} | \mathbf{1})$ and $p_l(\cdot | \hat{X}) = p_c(\cdot | \hat{X})$, meaning that the data distribution created by generator $G$ and the label distribution predicted by classifier $C$ will align with the actual data distribution and label distribution.
\end{proof}

\section*{Experiment}

This section evaluates and compares the imputation performance of the proposed SEGAN algorithm with 12 existing imputation models.
All algorithms are implemented in the Python environment, and all experiments are conducted on an Intel Core server with a 2.80GHz processor, equipped with a TITAN Xp 12GiB GPU and 192GB RAM.

\subsection*{Experiment Setup}
\textbf{Dataset}.
To thoroughly verify the effectiveness of the SEGAN algorithm under real and complex missing data scenarios, three real-world incomplete datasets were used in the experiments of this chapter. (1) The COVID-19 Trial Tracker (Tria) dataset ~\cite{bib41} presents clinical trial data on COVID-19 research and tracks the availability of study outcomes. The Tria dataset records 9 features of 6,433 trials and contains a 9.63\% missing data rate. (2) The Emergency Declarations Timeline (Emer) dataset ~\cite{bib40} chronicles the emergency declarations and mitigation policies of various U.S. states starting from January 20, 2020, comprising 8,364 samples with 22 features and a 62.69\% missing data rate. (3) The Government Response Timeline (Resp) dataset ~\cite{bib39} compiles government responses starting from January 1, 2020, including 200,737 samples with 19 features, among which 5.66\% of the data is missing.

\textbf{Metrices}.
In the experiments, Root Mean Squared Error (RMSE) is employed to measure the imputation accuracy of the imputation models. A lower RMSE value indicates better model performance. Additionally, 20\% of the known data is randomly deleted during the experiments, and the RMSE is calculated based on these data. Each evaluation metric value refers to the average of the results from five experiments under different random data splits. Moreover, in the experiments, min-max normalization is applied to numerical features on each dataset to prevent certain dimensions from dominating the model training, and one-hot encoding is used to represent categorical features.

\textbf{Baseline}.
The experiments include 12 existing state-of-the-art imputation methods for comparison, which comprise 5 machine learning-based imputation methods: MissForest (MissF), XGBoost Imputer (XGBI), Bayesian PCA (Baran), Multiple Imputation by Chained Equations (MICE), and Iterative Collaborative Filtering for Recommender Systems (ICLR); 2 imputation methods based on Multilayer Perceptrons: DataWig and Robust Recursive Imputation (RRSI); 5 imputation methods based on autoencoders: Multiple Imputation using Denoising Autoencoders (MIDAE), Variational Autoencoder Imputer (VAEI), Multiple Imputation with Wasserstein Autoencoders (MIWAE), Efficient Data Imputation with Decision trees (EDDI), and Heteroscedastic Imputation Variational Autoencoder (HIVAE); 2 imputation models based on Generative Adversarial Networks: Graph Imputation Neural Network (GINN) and Generative Adversarial Imputation Nets (GAIN); and the imputation algorithm optimization framework EDIT.

\textbf{Implementation details}.
The source code for all data imputation models originates from the corresponding publications. For all machine learning-based imputation methods, the learning rate is set to 0.3, with the number of iterations set at 100. The MissForest (MissFI) algorithm is configured with 6 decision trees. Baran employs AdaBoost as the predictive model. The Multiple Imputation by Chained Equations (MICE) algorithm performs imputation 20 times. For all deep learning-based imputation algorithms, the learning rate is set to 0.001, with a dropout rate of 0.5, training epochs are set to 30, batch size at 128, and the ADAM optimizer is used for network training. The Multiple Imputation using Denoising Autoencoders (MIDAE) algorithm consists of 2 fully connected layers with 128 neural units each. The encoder and decoder of the Variational Autoencoder Imputer (VAEI) are fully connected networks, each module having two hidden layers with 20 neurons each, and the latent vector dimensionality is 10. In the Graph Imputation Neural Network (GINN) algorithm, the discriminator is a simple 3-layer feedforward network. In the Generative Adversarial Imputation Nets (GAIN) algorithm, both the generator and discriminator are modeled as 2-layer fully connected networks. The user-specified parameter $\alpha$ for the EDIT-GAIN algorithm is set to 100\%. Additionally, in the SEGAN algorithm, the discriminator's time hint matrix encodes 80\% of the mask matrix information, with the default data containing a label rate of 100\%.

\begin{table}[t]
\centering
\begin{tabular}{lccc}
\hline
Method  & Tria  & Emer  & Resp  \\
\hline
MissFI  & 0.417 & 0.377 & ---   \\
XGBI    & 0.418 & 0.372 & ---   \\
Baran   & 0.412 & ---   & ---   \\
MICE    & 0.402 & ---   & ---   \\
ICLR    & 0.411 & ---   & ---   \\
DataWig & 0.458 & 0.389 & ---   \\
RRSI    & 0.398 & 0.358 & ---   \\
MIDAE   & 0.445 & 0.378 & ---   \\
MIWAE   & 0.396 & ---   & ---   \\
HIVAE   & 0.396 & 0.357 & 0.398 \\
GAIN    & 0.398 & 0.352 & 0.396 \\
EDIT    & 0.384 & 0.341 & 0.386 \\
SEGAN   & \textbf{0.376} & \textbf{0.332} & \textbf{0.380} \\
\hline
\end{tabular}
\caption{Comparison of methods on incomplete datasets.}
\label{table1}
\end{table}

\section*{Result}
Table ~\ref{table1} presents the results of missing feature imputation experiments on three incomplete datasets using various data imputation methods. It should be noted that some imputation methods were unable to complete model training within 24 hours, resulting in missing outcomes, which are denoted by ``$-$''.

It is evident that the SEGAN algorithm significantly outperforms all comparison algorithms in terms of imputation performance. In terms of imputation accuracy (i.e., RMSE), the SEGAN algorithm achieves an average improvement of 2.09\% over the best comparison algorithm, EDIT, and even a 2.64\% increase in imputation accuracy over the CAEGAN algorithm on the Emer dataset. This improvement can be attributed to the SEGAN algorithm's consideration and analysis of data label information, which enhances the model's imputation accuracy and stability.
It is important to note that due to the high complexity of models such as MissForest (MissFI), XGBoost Imputer (XGBI), Bayesian PCA (Baran), Multiple Imputation by Chained Equations (MICE), Iterative Collaborative Filtering for Recommender Systems (ICLR), DataWig, Robust Recursive Imputation (RRSI), Multiple Imputation using Denoising Autoencoders (MIDAE), and Multiple Imputation with Wasserstein Autoencoders (MIWAE), they could not complete model training within 24 hours. Consequently, these algorithms will not be compared in subsequent experiments.

\subsection*{Sensitivity Assessment of Missingness Rate}

Figure ~\ref{fig2} illustrates the RMSE results of the imputation algorithms as the data missingness rate Rm (i.e., the proportion of missing data in each dataset) increases from 20\% to 80\%. It is observed that in all cases, the SEGAN algorithm achieves better imputation accuracy, indicated by lower prediction errors (RMSE), compared to the existing optimal imputation algorithms. As the missingness rate Rm increases, the imputation performance of the SEGAN algorithm remains more robust. 

\begin{figure}[t]
    \centering
    \begin{subfigure}{0.32\textwidth}
        \includegraphics[width=\linewidth]{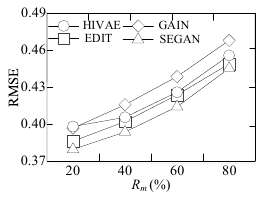}
        \caption{Tria}
        \label{fig:sub1}
    \end{subfigure}\hfill
    \begin{subfigure}{0.32\textwidth}
        \includegraphics[width=\linewidth]{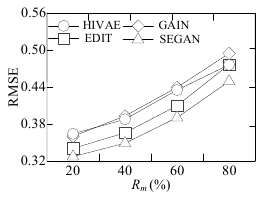}
        \caption{Emer}
        \label{fig:sub2}
    \end{subfigure}\hfill
    \begin{subfigure}{0.32\textwidth}
        \includegraphics[width=\linewidth]{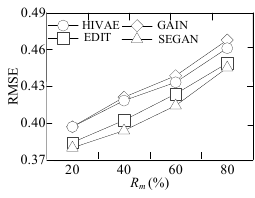}
        \caption{Resp}
        \label{fig:sub3}
    \end{subfigure}
    \caption{The prediction performance vs. missing rate $R_m$.}
    \label{fig2}
\end{figure}

Furthermore, the imputation accuracy of all algorithms tends to decline with the increase in the missingness rate. The reason for this is that as the missingness rate rises, the amount of known data available to the imputation algorithms decreases, leading to a corresponding decline in model imputation accuracy.

\subsection*{Ablation Study}

In this experiment, we explored the impact of different components within the SEGAN algorithm on its imputation performance. Table~\ref{table2} presents the RMSE for the SEGAN algorithm across various datasets. The S-no-C algorithm refers to SEGAN without the classifier, S-no-D denotes SEGAN without the discriminator, and S-no-R represents SEGAN without the time hint matrix.

\begin{table}[t]
\centering
\begin{tabular}{lccc}
\toprule
Method & Tria & Emer & Resp \\
\midrule
S-no-C & 0.380 & 0.340 & 0.384 \\
S-no-D & 0.395 & 0.350 & 0.393 \\
S-no-R & 0.385 & 0.340 & 0.385 \\
SEGAN  & \textbf{0.376} & \textbf{0.332} & \textbf{0.380} \\
\bottomrule
\end{tabular}
\caption{The ablation study of SEGAN}
\label{table2}
\end{table}

It was found that each component of the SEGAN algorithm, namely the classifier, discriminator, and time hint matrix, contributes to varying degrees in enhancing the imputation performance of the SEGAN algorithm. Specifically, when the classifier, discriminator, or time hint matrix is removed, the imputation accuracy of the SEGAN algorithm decreases by an average of 1.51\%, 4.63\%, and 2.35\%, respectively. Therefore, the discriminator within the SEGAN algorithm has the most significant impact on the model's imputation performance.

\subsection*{Post-Imputation Predictive Evaluation}

 The final set of experiments validated the performance of the SEGAN algorithm on predictive tasks following data imputation. Table ~\ref{table3} shows the experimental results of the imputation algorithms on the classification task for the Tria dataset and the regression tasks for the Emer and Resp datasets. A higher AUC value indicates better predictive performance, while the opposite is true for MAE. In the experiments, the imputation methods were first used to repair incomplete datasets; then, a regression/classification model composed of three fully connected layers was trained on the imputed data. The number of training epochs was set to 30, the learning rate was set to 0.005, the dropout rate was set to 0.5, and the batch size was set to 128.

\begin{table}[ht]
\centering
\begin{tabular}{lccc}
\toprule
Method & \multicolumn{3}{c}{Tria} \\
       & AUC   & RMSE    & RMSE \\
\midrule
HIVAE  & 0.896 & 123.232 & 107.382 \\
GAIN   & 0.903 & 122.312 & 106.216 \\
EDIT   & 0.912 & 121.223 & 105.493 \\
SEGAN  & \textbf{0.923} & \textbf{120.683} & \textbf{104.392} \\
\bottomrule
\end{tabular}
\caption{Post-imputation evaluation on imputation methods}
\label{table3}
\end{table}

It was observed that the predictive accuracy of the SEGAN algorithm surpasses all comparison algorithms. Specifically, compared to the best-performing comparison algorithm, EDIT, the SEGAN algorithm reduced the MAE by an average of 0.75\% on regression tasks and increased the AUC by an average of 1.21\% on classification tasks. Moreover, the SEGAN algorithm showed the most significant improvement on the classification task for the Tria dataset compared to the EDIT algorithm. This further confirms the effectiveness of the SEGAN algorithm.

\section*{Discussion and Conclusions}

In response to the issue of incomplete data and label information, this paper proposes a semi-supervised learning-based method for missing data imputation named SEGAN, which includes three modules: a generator, a discriminator, and a classifier. The SEGAN algorithm introduces a timing hint matrix to the discriminator to better distinguish between generated and real data and proposes a semi-supervised classifier to address the problem of insufficient data labels. Additionally, this chapter also demonstrates that with the aid of the timing hint matrix and the semi-supervised classifier, the SEGAN algorithm can achieve a unique Nash equilibrium, allowing the generated data distribution and predicted label distribution to converge respectively to the true data distribution and the true label distribution. Ultimately, through extensive experimentation, it is proven that the SEGAN algorithm significantly outperforms existing state-of-the-art methods for multivariate time-series data imputation on both data imputation tasks and downstream prediction tasks.

\section*{Acknowledgments}
This research was funded by the National Key Research and Development Program, China (No. 2023YFC3306303), the Major scientific and technological projects of China National Tobacco Corporation(No. 110202201001(JY -01)), and the Zhejiang Provincial “Jianbing” Lingyan” Re-search and Development Program, China(No. 2023C01213). \\Supported by State Key Laboratory of Massive Personalized Customization System and Technology (H\&C-MPC-2023-05-02).

\begin{refcontext}[sorting = none]
\printbibliography

@article{bib1,
  title={Centrifugal blower fault trend prediction method based on Informer with incomplete data},
  author={You, Zhang and Congbo, Li and Lihong, Lin and Jing, Qian and Qian, Yi},
  journal={Computer Integrated Manufacturing System},
  volume={29},
  number={1},
  pages={133},
  year={2023}
}

@article{bib2,
  title={A novel framework for imputation of missing values in databases},
  author={Farhangfar, Alireza and Kurgan, Lukasz A and Pedrycz, Witold},
  journal={IEEE Transactions on Systems, Man, and Cybernetics-Part A: Systems and Humans},
  volume={37},
  number={5},
  pages={692--709},
  year={2007},
  publisher={IEEE}
}

@article{bib3,
  title={An introduction to kernel and nearest-neighbor nonparametric regression},
  author={Altman, Naomi S},
  journal={The American Statistician},
  volume={46},
  number={3},
  pages={175--185},
  year={1992},
  publisher={Taylor \& Francis}
}

@article{bib4,
  title={Missing data imputation using statistical and machine learning methods in a real breast cancer problem},
  author={Jerez, Jos{\'e} M and Molina, Ignacio and Garc{\'\i}a-Laencina, Pedro J and Alba, Emilio and Ribelles, Nuria and Mart{\'\i}n, Miguel and Franco, Leonardo},
  journal={Artificial intelligence in medicine},
  volume={50},
  number={2},
  pages={105--115},
  year={2010},
  publisher={Elsevier}
}

@inproceedings{bib5,
  title={Xgboost: A scalable tree boosting system},
  author={Chen, Tianqi and Guestrin, Carlos},
  booktitle={Proceedings of the 22nd acm sigkdd international conference on knowledge discovery and data mining},
  pages={785--794},
  year={2016}
}

@article{bib6,
  title={MissForest—non-parametric missing value imputation for mixed-type data},
  author={Stekhoven, Daniel J and B{\"u}hlmann, Peter},
  journal={Bioinformatics},
  volume={28},
  number={1},
  pages={112--118},
  year={2012},
  publisher={Oxford University Press}
}

@article{bib7,
  title={Multiple imputation by chained equations (MICE): implementation in Stata},
  author={Royston, Patrick and White, Ian R},
  journal={Journal of statistical software},
  volume={45},
  pages={1--20},
  year={2011}
}

@inproceedings{bib8,
  title={Learning individual models for imputation},
  author={Zhang, Aoqian and Song, Shaoxu and Sun, Yu and Wang, Jianmin},
  booktitle={2019 IEEE 35th International Conference on Data Engineering (ICDE)},
  pages={160--171},
  year={2019},
  organization={IEEE}
}

@article{bib9,
  title={Spectral regularization algorithms for learning large incomplete matrices},
  author={Mazumder, Rahul and Hastie, Trevor and Tibshirani, Robert},
  journal={The Journal of Machine Learning Research},
  volume={11},
  pages={2287--2322},
  year={2010},
  publisher={JMLR. org}
}

@article{bib10,
  title={Algorithms for non-negative matrix factorization},
  author={Lee, Daniel and Seung, H Sebastian},
  journal={Advances in neural information processing systems},
  volume={13},
  year={2000}
}

@article{bib11,
  title={Multiple imputation in principal component analysis},
  author={Josse, Julie and Pag{\`e}s, J{\'e}r{\^o}me and Husson, Fran{\c{c}}ois},
  journal={Advances in data analysis and classification},
  volume={5},
  pages={231--246},
  year={2011},
  publisher={Springer}
}

@article{bib12,
  title={Pattern classification with missing data: a review},
  author={Garc{\'\i}a-Laencina, Pedro J and Sancho-G{\'o}mez, Jos{\'e}-Luis and Figueiras-Vidal, An{\'\i}bal R},
  journal={Neural Computing and Applications},
  volume={19},
  pages={263--282},
  year={2010},
  publisher={Springer}
}

@inproceedings{bib13,
  title={Missing data imputation using optimal transport},
  author={Muzellec, Boris and Josse, Julie and Boyer, Claire and Cuturi, Marco},
  booktitle={International Conference on Machine Learning},
  pages={7130--7140},
  year={2020},
  organization={PMLR}
}

@inproceedings{bib14,
  title={Mida: Multiple imputation using denoising autoencoders},
  author={Gondara, Lovedeep and Wang, Ke},
  booktitle={Advances in Knowledge Discovery and Data Mining: 22nd Pacific-Asia Conference, PAKDD 2018, Melbourne, VIC, Australia, June 3-6, 2018, Proceedings, Part III 22},
  pages={260--272},
  year={2018},
  organization={Springer}
}

@article{bib15,
  title={Variational autoencoders for missing data imputation with application to a simulated milling circuit},
  author={McCoy, John T and Kroon, Steve and Auret, Lidia},
  journal={IFAC-PapersOnLine},
  volume={51},
  number={21},
  pages={141--146},
  year={2018},
  publisher={Elsevier}
}

@article{bib16,
  title={Handling incomplete heterogeneous data using vaes},
  author={Nazabal, Alfredo and Olmos, Pablo M and Ghahramani, Zoubin and Valera, Isabel},
  journal={Pattern Recognition},
  volume={107},
  pages={107501},
  year={2020},
  publisher={Elsevier}
}

@inproceedings{bib17,
  title={MIWAE: Deep generative modelling and imputation of incomplete data sets},
  author={Mattei, Pierre-Alexandre and Frellsen, Jes},
  booktitle={International conference on machine learning},
  pages={4413--4423},
  year={2019},
  organization={PMLR}
}

@inproceedings{bib18,
  title={Gain: Missing data imputation using generative adversarial nets},
  author={Yoon, Jinsung and Jordon, James and Schaar, Mihaela},
  booktitle={International conference on machine learning},
  pages={5689--5698},
  year={2018},
  organization={PMLR}
}

@article{bib19,
  title={Missing data imputation with adversarially-trained graph convolutional networks},
  author={Spinelli, Indro and Scardapane, Simone and Uncini, Aurelio},
  journal={Neural Networks},
  volume={129},
  pages={249--260},
  year={2020},
  publisher={Elsevier}
}

@article{bib20,
  title={Differentiable and scalable generative adversarial models for data imputation},
  author={Wu, Yangyang and Wang, Jun and Miao, Xiaoye and Wang, Wenjia and Yin, Jianwei},
  journal={IEEE Transactions on Knowledge and Data Engineering},
  year={2023},
  publisher={IEEE}
}

@article{bib21,
  title={Efficient and effective data imputation with influence functions},
  author={Miao, Xiaoye and Wu, Yangyang and Chen, Lu and Gao, Yunjun and Wang, Jun and Yin, Jianwei},
  journal={Proceedings of the VLDB Endowment},
  volume={15},
  number={3},
  pages={624--632},
  year={2021},
  publisher={VLDB Endowment}
}

@article{bib22,
  title={Missing value estimating algorithm based on cloud manufacturing services QoS time series data properties},
  author={Li, Shan and Yu, Ying and Hu, Kanghua and Song, Bo and Yao, Yehui and Yin, Jianwei},
  journal={Computer Integrated Manufacturing Systems},
  volume={12},
  number={21},
  pages={2930--2936},
  year={2016}
}

@article{bib23,
  title={PhysioBank, PhysioToolkit, and PhysioNet: components of a new research resource for complex physiologic signals},
  author={Goldberger, Ary L and Amaral, Luis AN and Glass, Leon and Hausdorff, Jeffrey M and Ivanov, Plamen Ch and Mark, Roger G and Mietus, Joseph E and Moody, George B and Peng, Chung-Kang and Stanley, H Eugene},
  journal={circulation},
  volume={101},
  number={23},
  pages={e215--e220},
  year={2000},
  publisher={Am Heart Assoc}
}

@article{bib24,
  title={Generative adversarial nets},
  author={Goodfellow, Ian and Pouget-Abadie, Jean and Mirza, Mehdi and Xu, Bing and Warde-Farley, David and Ozair, Sherjil and Courville, Aaron and Bengio, Yoshua},
  journal={Advances in neural information processing systems},
  volume={27},
  year={2014}
}

@article{bib25,
  title={An experimental survey of missing data imputation algorithms},
  author={Miao, Xiaoye and Wu, Yangyang and Chen, Lu and Gao, Yunjun and Yin, Jianwei},
  journal={IEEE Transactions on Knowledge and Data Engineering},
  year={2022},
  publisher={IEEE}
}

@article{bib26,
  title={Missing value imputation in multivariate time series with end-to-end generative adversarial networks},
  author={Zhang, Ying and Zhou, Baohang and Cai, Xiangrui and Guo, Wenya and Ding, Xiaoke and Yuan, Xiaojie},
  journal={Information Sciences},
  volume={551},
  pages={67--82},
  year={2021},
  publisher={Elsevier}
}

@article{bib27,
  title={On the choice of the best imputation methods for missing values considering three groups of classification methods},
  author={Luengo, Juli{\'a}n and Garc{\'\i}a, Salvador and Herrera, Francisco},
  journal={Knowledge and information systems},
  volume={32},
  pages={77--108},
  year={2012},
  publisher={Springer}
}

@article{bib28,
  title={Missing data: Five practical guidelines},
  author={Newman, Daniel A},
  journal={Organizational Research Methods},
  volume={17},
  number={4},
  pages={372--411},
  year={2014},
  publisher={Sage Publications Sage CA: Los Angeles, CA}
}

@article{bib29,
  title={Comparison of performance of data imputation methods for numeric dataset},
  author={Jadhav, Anil and Pramod, Dhanya and Ramanathan, Krishnan},
  journal={Applied Artificial Intelligence},
  volume={33},
  number={10},
  pages={913--933},
  year={2019},
  publisher={Taylor \& Francis}
}

@article{bib30,
  title={Missing data imputation: focusing on single imputation},
  author={Zhang, Zhongheng},
  journal={Annals of translational medicine},
  volume={4},
  number={1},
  year={2016},
  publisher={AME Publications}
}

@article{bib31,
  title={A benchmark for data imputation methods},
  author={J{\"a}ger, Sebastian and Allhorn, Arndt and Bie{\ss}mann, Felix},
  journal={Frontiers in big Data},
  volume={4},
  pages={693674},
  year={2021},
  publisher={Frontiers Media SA}
}

@article{bib32,
  title={From predictive methods to missing data imputation: an optimization approach},
  author={Bertsimas, Dimitris and Pawlowski, Colin and Zhuo, Ying Daisy},
  journal={Journal of Machine Learning Research},
  volume={18},
  number={196},
  pages={1--39},
  year={2018}
}

@article{bib33,
  title={Deep learning versus conventional methods for missing data imputation: A review and comparative study},
  author={Sun, Yige and Li, Jing and Xu, Yifan and Zhang, Tingting and Wang, Xiaofeng},
  journal={Expert Systems with Applications},
  pages={120201},
  year={2023},
  publisher={Elsevier}
}

@article{bib34,
  title={Comparison of imputation methods for missing laboratory data in medicine},
  author={Waljee, Akbar K and Mukherjee, Ashin and Singal, Amit G and Zhang, Yiwei and Warren, Jeffrey and Balis, Ulysses and Marrero, Jorge and Zhu, Ji and Higgins, Peter DR},
  journal={BMJ open},
  volume={3},
  number={8},
  pages={e002847},
  year={2013},
  publisher={British Medical Journal Publishing Group}
}

@article{bib35,
  title={Comparison of missing data imputation methods in time series forecasting},
  author={Ahn, Hyun and Sun, Kyunghee and Kim, K Pio and others},
  journal={Computers, Materials \& Continua},
  volume={70},
  number={1},
  pages={767--779},
  year={2022}
}

@article{bib36,
  title={Study on the missing data mechanisms and imputation methods},
  author={Alruhaymi, Abdullah Z and Kim, Charles J},
  journal={Open Journal of Statistics},
  volume={11},
  number={4},
  pages={477--492},
  year={2021},
  publisher={Scientific Research Publishing}
}

@article{bib37,
  title={Handling missing data with graph representation learning},
  author={You, Jiaxuan and Ma, Xiaobai and Ding, Yi and Kochenderfer, Mykel J and Leskovec, Jure},
  journal={Advances in Neural Information Processing Systems},
  volume={33},
  pages={19075--19087},
  year={2020}
}

@article{bib38,
  title={Reviewing autoencoders for missing data imputation: Technical trends, applications and outcomes},
  author={Pereira, Ricardo Cardoso and Santos, Miriam Seoane and Rodrigues, Pedro Pereira and Abreu, Pedro Henriques},
  journal={Journal of Artificial Intelligence Research},
  volume={69},
  pages={1255--1285},
  year={2020}
}

@article{bib39,
  title={A structured open dataset of government interventions in response to COVID-19},
  author={Desvars-Larrive, Amelie and Dervic, Elma and Haug, Nina and Niederkrotenthaler, Thomas and Chen, Jiaying and Di Natale, Anna and Lasser, Jana and Gliga, Diana S and Roux, Alexandra and Sorger, Johannes and others},
  journal={Scientific data},
  volume={7},
  number={1},
  pages={285},
  year={2020},
  publisher={Nature Publishing Group UK London}
}

@misc{bib40,
  author = {{Federal Emergency Management Agency}},
  title = {OpenFEMA Dataset: OpenFEMA Data Set Fields},
  howpublished = {\url{https://www.fema.gov/api/open/v1/OpenFemaDataSetFields}},
  note = {Accessed: 2024-05-28},
  year = {2024},
  urldate = {2024-05-28},
  url = {https://www.fema.gov/about/reports-and-data/openfema},
}

@misc{bib41,
  title = {COVID-19 TrialsTracker},
  author = {DeVito, Nicholas and Aronson, Jeffrey K. and Drysdale, Henry and Liu, Michael and Curtis, Helen and Heneghan, Carl and Ferner, Robin and others},
  organization = {Bennett Institute for Applied Data Science and The Centre for Evidence-Based Medicine, University of Oxford},
  howpublished = {\url{https://covid19.trialstracker.net/}},
  note = {Accessed: 2024-05-28},
  year = {2024},
}

@article{bib42,
  title={Comparison of results from different imputation techniques for missing data from an anti-obesity drug trial},
  author={J{\o}rgensen, Anders W and Lundstr{\o}m, Lars H and Wetterslev, J{\o}rn and Astrup, Arne and G{\o}tzsche, Peter C},
  journal={PLoS One},
  volume={9},
  number={11},
  pages={e111964},
  year={2014},
  publisher={Public Library of Science San Francisco, USA}
}

@article{bib43,
  title={Data driven estimation of imputation error—a strategy for imputation with a reject option},
  author={Bak, Nikolaj and Hansen, Lars K},
  journal={PLoS One},
  volume={11},
  number={10},
  pages={e0164464},
  year={2016},
  publisher={Public Library of Science San Francisco, CA USA}
}

@article{bib44,
  title={On combining reference data to improve imputation accuracy},
  author={Chen, Jun and Zhang, Ji-Gang and Li, Jian and Pei, Yu-Fang and Deng, Hong-Wen},
  journal={PLoS One},
  volume={8},
  number={1},
  pages={e55600},
  year={2013},
  publisher={Public Library of Science San Francisco, USA}
}

@article{bib45,
  title={Effects of different missing data imputation techniques on the performance of undiagnosed diabetes risk prediction models in a mixed-ancestry population of South Africa},
  author={Masconi, Katya L and Matsha, Tandi E and Erasmus, Rajiv T and Kengne, Andre P},
  journal={PloS one},
  volume={10},
  number={9},
  pages={e0139210},
  year={2015},
  publisher={Public Library of Science San Francisco, CA USA}
}

@article{bib46,
  title={A new analytical framework for missing data imputation and classification with uncertainty: Missing data imputation and heart failure readmission prediction},
  author={Hu, Zhiyong and Du, Dongping},
  journal={PloS one},
  volume={15},
  number={9},
  pages={e0237724},
  year={2020},
  publisher={Public Library of Science San Francisco, CA USA}
}

@article{bib47,
  title={Missing data in randomized clinical trials for weight loss: scope of the problem, state of the field, and performance of statistical methods},
  author={Elobeid, Mai A and Padilla, Miguel A and McVie, Theresa and Thomas, Olivia and Brock, David W and Musser, Bret and Lu, Kaifeng and Coffey, Christopher S and Desmond, Renee A and St-Onge, Marie-Pierre and others},
  journal={PloS one},
  volume={4},
  number={8},
  pages={e6624},
  year={2009},
  publisher={Public Library of Science San Francisco, USA}
}

@article{bib48,
  title={Inconsistent definitions for intention-to-treat in relation to missing outcome data: systematic review of the methods literature},
  author={Alshurafa, Mohamad and Briel, Matthias and Akl, Elie A and Haines, Ted and Moayyedi, Paul and Gentles, Stephen J and Rios, Lorena and Tran, Chau and Bhatnagar, Neera and Lamontagne, Francois and others},
  journal={PLOS one},
  volume={7},
  number={11},
  pages={e49163},
  year={2012},
  publisher={Public Library of Science San Francisco, USA}
}

@article{bib49,
  title={Missing data imputation: focusing on single imputation},
  author={Zhang, Zhongheng},
  journal={Annals of translational medicine},
  volume={4},
  number={1},
  year={2016},
  publisher={AME Publications}
}

@article{bib50,
  title={A Kriging based spatiotemporal approach for traffic volume data imputation},
  author={Yang, Hongtai and Yang, Jianjiang and Han, Lee D and Liu, Xiaohan and Pu, Li and Chin, Shih-miao and Hwang, Ho-ling},
  journal={PloS one},
  volume={13},
  number={4},
  pages={e0195957},
  year={2018},
  publisher={Public Library of Science San Francisco, CA USA}
}
\end{refcontext}

\end{document}